\documentclass[letterpaper, 10 pt, conference]{ieeeconf} 
\IEEEoverridecommandlockouts                              
\usepackage{graphicx}
\usepackage[letterpaper, left=0.68in, right=0.68in, top=0.75in, bottom=0.75in]{geometry}
\usepackage{svg}
\usepackage{listings}
\usepackage{subcaption}

\usepackage{etoolbox}

\usepackage{scalerel}

\usepackage[htt]{hyphenat}

\usepackage{siunitx}

\usepackage{pifont}

\usepackage{amsmath,textcomp}

\usepackage{tikz}
\usetikzlibrary{shapes, arrows, positioning, fit}
\usepackage{comment}
\allowdisplaybreaks
\usepackage{adjustbox}

\usepackage{url}

\usepackage{algpseudocode}
\usepackage{paralist}
\usepackage{amsthm}
\usepackage{color}
\usepackage{stmaryrd}
\usepackage{epsfig}
\usepackage{bm}
\usepackage{listings,lstautogobble}
\usepackage{mathtools}
\usepackage{mathrsfs}
\usepackage{amsfonts}
\usepackage{newpxtext,newpxmath}
\usepackage{xspace}
\usepackage{verbatim}
\usepackage{epstopdf}
\usepackage{parcolumns}

\usepackage[utf8]{inputenc}
\usepackage{calrsfs}
\usepackage{rotating}
\usepackage{booktabs}
\usepackage{multirow}

\usepackage{makecell}
\usepackage{float}
\newcommand{\mycomment
}[1]{}

\usepackage[compact]{titlesec}
\titlespacing{\section}{0pt}{0.25ex}{0.25ex}
\titlespacing{\subsection}{0pt}{0.2ex}{0.2ex}
\titlespacing{\subsubsection}{0pt}{0.1ex}{0.1ex}

\usepackage[T1]{fontenc}

\DeclareMathAlphabet{\mathcal}{OMS}{cmsy}{m}{n}
\newcommand{\regtext}[1]{\mathrm{\textnormal{#1}}}

\newcommand{\R}{{\mathbb{R}}}

\newcommand{\N}{{\mathbb{N}}}

\DeclarePairedDelimiter{\norm}{\lVert}{\rVert}

\DeclareMathOperator*{\argmin}{arg\,min}

\newcommand{\opt}{^\star}
\newcommand{\proj}{\regtext{proj}}

\newcommand{\vc}[1]{{\mathbf{#1}}}

\newcommand{\arridx}[2]{{\left(#1\right)}_{#2}}

\newcommand{\lbl}[1]{_{\regtext{#1}}}

\newcommand{\goal}{\lbl{goal}}

\newcommand{\action}{\vc{u}}

\newcommand{\RLaction}{\action}

\newtheorem{definition}{Definition}
 \newtheorem{theorem}{Theorem}

\usepackage{algorithm}
\usepackage{amsmath}
\usepackage{algpseudocode}

  \renewcommand{\footnotesize}{\fontsize{8pt}{11pt}\selectfont}

\DeclareMathAlphabet{\mathcal}{OMS}{cmsy}{m}{n}

\usepackage{hyperref}
\hypersetup{
    colorlinks=true,
    linkcolor=black,
    urlcolor=gray,
}
\usepackage{balance}

\newcommand{\diag}{\text{diag}}

\newcommand{\setdef}[2][]{
	\left\{
	\ifblank{#1}{}{#1 \hspace{.1cm} \middle| \hspace{.1cm}}
	#2
	\right\}
}

\makeatletter
\newcommand{\vast}{\bBigg@{4}}
\newcommand{\Vast}{\bBigg@{5}}
\makeatother

\makeatletter
\DeclareRobustCommand{\nand}{\mathbin{\mathpalette\n@and@or\land}}
\DeclareRobustCommand{\nor}{\mathbin{\mathpalette\n@and@or\lor}}

\DeclareRobustCommand{\enand}{\overline{\mathbin{\mathpalette\n@and@or\land}}}
\DeclareRobustCommand{\enor}{\overline{\mathbin{\mathpalette\n@and@or\lor}}}

\newcommand{\n@and@or}[2]{
  \vphantom{#2}
  \ooalign{$\m@th#1#2$\cr\hidewidth$\m@th#1\sim$\hidewidth\cr}
}
\makeatother

\begin{document}
\title{\LARGE \bf
Safe LLM-Controlled Robots with Formal Guarantees \\ via Reachability Analysis}
\author{
Ahmad Hafez$^{*1}$, Alireza Naderi Akhormeh$^{*1}$, Amr Hegazy$^{2}$, and Amr Alanwar$^{1}$
\vspace{-5cm}
\thanks{$^{*}$ Authors are with equal contributions.}
\thanks{$^{1}$ Authors are with the Technical University of Munich; TUM School of Computation, Information and Technology, Department of Computer Engineering. {\tt\small \{a.hafez, alireza.naderi, alanwar\}@tum.de}.}
\thanks{$^{2}$ Author is with the German university in Cairo; Faculty of Media Engineering and Technology, Department of Computer Science and Engineering. {\tt\small {amr.hazem}@student.guc.edu.eg}.}}

\maketitle

\begin{abstract}

The deployment of Large Language Models (LLMs) in robotic systems presents unique safety challenges, particularly in unpredictable environments. Although LLMs, leveraging zero-shot learning, enhance human-robot interaction and decision-making capabilities, their inherent probabilistic nature and lack of formal guarantees raise significant concerns for safety-critical applications. Traditional model-based verification approaches often rely on precise system models, which are difficult to obtain for real-world robotic systems and may not be fully trusted due to modeling inaccuracies, unmodeled dynamics, or environmental uncertainties. To address these challenges, this paper introduces a safety assurance framework for LLM-controlled robots based on data-driven reachability analysis, a formal verification technique that ensures all possible system trajectories remain within safe operational limits. Our framework specifically investigates the problem of instructing an LLM to navigate the robot to a specified goal and assesses its ability to generate low-level control actions that successfully guide the robot safely toward that goal. By leveraging historical data to construct reachable sets of states for the robot-LLM system, our approach provides rigorous safety guarantees against unsafe behaviors without relying on explicit analytical models. We validate the framework through experimental case studies in autonomous navigation and task planning, demonstrating its effectiveness in mitigating risks associated with LLM-generated commands. This work advances the integration of formal methods into LLM-based robotics, offering a principled and practical approach to ensuring safety in next-generation autonomous systems.

\end{abstract}

\section{INTRODUCTION}

Integrating the Large Language Models (LLMs) into robotics has enabled robots to interpret natural language commands, adapt to unstructured environments, and collaborate with humans in transformative ways~\cite{wu2024safety}. Applications range from assistive healthcare robots to autonomous delivery systems, where LLMs act as high-level controllers to translate human intent into robotic actions. However, as these systems transition to real-world deployment, their safety and reliability remain critical challenges. Unlike traditional rule-based controllers, LLMs generate outputs through probabilistic reasoning, introducing uncertainties that are difficult to model or verify. Recent studies highlight that even minor adversarial modifications to input prompts or perceptual data can degrade system performance by 19–29\%, underscoring the risks of deploying LLMs in safety-critical scenarios~\cite{wu2024safety}.

Existing safety assurance methods, such as runtime monitoring or constraint-based control, often assume deterministic decision-making models and struggle to address the open-ended behavior of LLMs~\cite{yang2024plug}. For example, LLM-controlled robots may misinterpret ambiguous instructions (e.g., "avoid obstacles ahead") or fail to recognize different hazards not covered in training data~\cite{robey2024jailbreaking,wu2024safety}. This gap is exacerbated by the lack of formal verification frameworks tailored to systems where LLM-generated decisions interact with robotic dynamics.

\subsection{Related Work}
Integrating LLMs into robotics has opened new possibilities for adaptive and intelligent robotic systems. However, this integration also introduces significant challenges, particularly in ensuring safety and reliability. This section reviews recent advancements in LLM-driven robotic control, identifies safety challenges unique to LLM-controlled systems, and examines existing approaches to safety assurance using formal methods. We also discuss the role of reachability analysis in robotics and highlight the gaps in current research, setting the stage for our proposed framework.

\subsubsection{LLM-Driven Robotic Control}
LLMs have emerged as powerful tools for robotic task planning and control. Recent works demonstrate their ability to generate low-level commands for dynamic locomotion~\cite{wang2023prompt}, decompose long-horizon tasks into multi-step plans~\cite{Ouyang2024}, and bridge high-level reasoning with low-level policies using latent codes~\cite{Shentu2024}. Hierarchical frameworks further optimize computational efficiency by decoupling high-frequency control from low-frequency semantic reasoning~\cite{zhang2024hirt}. However, these approaches prioritize flexibility and adaptability over formal safety guarantees, relying instead on post-hoc validation or empirical testing.

\subsubsection{Safety Challenges in LLM-Controlled Systems}
The integration of LLMs into robotics introduces safety risks. Studies show that vulnerabilities in grounding language instructions to physical actions can lead to issues in robot behavior under adversarial inputs~\cite{Myers2023}. The probabilistic nature of LLMs exacerbates these risks, as their outputs may unpredictably violate safety constraints in different environments~\cite{Bajcsy2024}. For instance, LLM-generated trajectories for humanoids can fail to account for temporal consistency, leading to unstable motions~\cite{Radosavovic2024}.

\subsubsection{Safety Assurance via Formal Methods}
Some efforts to ensure safety in LLM-driven systems employ formal verification techniques. Other approaches leverage LLMs to diagnose and repair unsafe motion planners, though their focus remains on traditional planning algorithms rather than LLM-generated policies~\cite{Lin2024}. Unlike Büchi automata, hybrid verification frameworks address sequence-aware safety but lack support for the open-ended decision-making of LLMs~\cite{wang2024ensuring}.

\subsubsection{Reachability Analysis in Robotics}
Reachability analysis has been widely adopted for safety-critical systems, enabling exhaustive verification of system trajectories. Recent work integrates reachability-based safety controllers with LLM-generated task plans~\cite{Thumm2024}. However, these approaches often assume a known system model, which may not be accurate in practice and could lead to safety violations, while other studies synthesize interpretable policies using LLM-guided search. However, existing methods either assume deterministic decision-making or focus on non-LLM systems~\cite{Lin2024,selim2022safe,MahmoudZPC}.

\subsubsection{Gaps and Novelty}

While LLMs enable unprecedented adaptability in robotics~\cite{Gupta2024,Javaid2024}, their integration with formal safety frameworks remains underexplored. Some approaches prioritize LLM flexibility without rigorous guarantees~\cite{wang2023prompt, Ouyang2024}.
Notably, many formal methods assume a precise robot model, which may not accurately reflect real-world dynamics and could lead to safety violations. In contrast, our work bridges this gap by proposing a reachability-based framework that treats LLM-controlled robots as dynamical systems, enabling formal verification of safety properties without relying on potentially inaccurate models. Our approach is inspired by~\cite{selim2022safe, chung2021constrained}, which we extended to preserve the adaptability of language models while ensuring robust safety guarantees.

This work's contributions are multifold:

\begin{itemize}
\item Unified safety framework with zero-shot learning for LLM-controlled systems:
We propose a novel framework that lets probabilistic language models control nonlinear dynamical systems and enable rigorous safety verification. The framework uniquely incorporates zero-shot learning, allowing it to generalize to unseen tasks and environments without task-specific training, significantly enhancing its versatility and applicability.

 \item Utilizing data-driven reachability analysis: 
 Our framework utilizes reachability analysis that eliminates the reliance on precise analytical models, which are often impractical or inaccurate, by leveraging historical data to construct robust reachable sets. This model-free approach provides robust safety guarantees in complex and uncertain environments, building on and extending prior work in data-driven reachability analysis~\cite{alanwar2023data}. The method is computationally efficient and scalable, making it suitable for real-world deployment in different settings.

 \item Performance evaluation: 
 Through simulations and real experiments, we validate the framework’s effectiveness in ensuring safety for LLM-controlled robots operating in different and unpredictable scenarios. The experiments evaluate key metrics such as scalability, adaptability, and real-time performance, demonstrating the framework’s practical utility. The results highlight the system’s ability to maintain robust safety guarantees across diverse and evolving conditions, showcasing its readiness for real-world applications.
\end{itemize}
Readers can watch the videos of the proposed approach on our YouTube playlist\footnote{\href{https://www.youtube.com/playlist?list=PLzH0T78uTTsuyMuDZ6bKfJPafIcgKau76}{\mbox{\textcolor{blue}{{http://tiny.cc/SafeLLMRA-Videos}}}}}, and reproduce our results by utilizing our openly accessible repository\footnote{\href{https://github.com/TUM-CPS-HN/SafeLLMRA}{\mbox{\textcolor{blue}{{https://github.com/TUM-CPS-HN/SafeLLMRA}}}}}.

The remainder of this paper is organized as follows: In Section~\ref{sec:prelim}, the preliminaries and problem statement are introduced. Section~\ref{sec:main} presents the proposed approach, while Section~\ref{sec:eval} presents experimental results. Lastly, Section~\ref{sec:con} concludes this paper with final remarks.

\section{Preliminaries and Problem Statement}\label{sec:prelim}

This section presents the notation, preliminary definitions, and the problem statement.

\subsection{Notation}
The set of $n$-dimensional real numbers is denoted by $\mathbb{R}^n$, the natural numbers by $\mathbb{N}$, and the set of integers from $n$ to $m$ by $n{:}m$. 
For a matrix ${A}$, the element at row $i$ and column $j$ is denoted by $({A})_{i,j}$, the $j$-th column by $({A})_{:\,,j}$. The $i$-th element of a vector ${a}$ by $\arridx{{a}}{i}$. 
A matrix or vector of ones with a proper dimension is represented as ${1}$.
We denote the Kronecker product by $\otimes$. 
The $\diag$ operator constructs a block-diagonal matrix by placing its arguments along the diagonal in a matrix of zeros. 
For sets $\mathcal{A}$ and $\mathcal{B}$, the Minkowski sum is defined as $\mathcal{A} + \mathcal{B} = \{ {a} + {b} \mid {a} \in \mathcal{A}, {b} \in \mathcal{B} \}$, and the Cartesian product as $\mathcal{A} \times \mathcal{B} = \left\{ \begin{bmatrix} a \\ b \end{bmatrix} \mid a \in \mathcal{A}, b \in \mathcal{B} \right\}$. Sets are represented using calligraphic font, e.g., $\mathcal{R}$. Infinity norm of $A$ is denoted by $\norm{A}_\infty$.

\subsection{Set Representations}

To represent sets, zonotopes and constrained zonotopes \cite{scott2016constrained} are employed, as they enable efficient computation of the Minkowski sum, a key operation in reachability analysis~\cite{althoff2010reachability}. They are introduced next.

\begin{definition}(\textbf{Zonotope} \cite{conf:zono1998}) \label{def:zonotopes} 
Given a center $c_{\mathcal{Z}} \in \mathbb{R}^{n}$ and $\gamma_{\mathcal{Z}} \in \mathbb{N}$ generator vectors in a generator matrix $G_{\mathcal{Z}}=\begin{bmatrix} g_{\mathcal{Z}}^{(1)} \dots g_{\mathcal{Z}}^{(\gamma_{\mathcal{Z}})}\end{bmatrix} \in \mathbb{R}^{n \times \gamma_{\mathcal{Z}}}$, a zonotope is defined as
\begin{equation}
	\mathcal{Z} = \Big\{ x \in \mathbb{R}^{n} \; \Big| \; x = c_{\mathcal{Z}} + G_{\mathcal{Z}} \beta \, ,
	\norm{\beta}_\infty\leq 1 \Big\} \; .
\end{equation}
We use the shorthand notation $\mathcal{Z} = \langle{c_{\mathcal{Z}},G_{\mathcal{Z}}}\rangle$ for a zonotope. 
\end{definition}

The \emph{Minkowski sum} of two zonotopes, $\mathcal{Z}_1 = \langle{{c}_1,{G}_1}\rangle$ and $\mathcal{Z}_2 = \langle{{c}_2,{G}_2}\rangle$, is computed as 
\[
\mathcal{Z}_1 + \mathcal{Z}_2 = \langle{{c}_1+{c}_2,[{G}_1,{G}_2]}\rangle
\] 
as established in \cite{althoff2010reachability}. Zonotopes have been generalized to represent any convex polytope by imposing constraints on the $\beta$ factors~\cite{scott2016constrained}. Compared to polyhedral sets, constrained zonotopes offer a key benefit: they retain the superior scalability of zonotopes as state-space dimensions grow, owing to their reliance on a generator-based set representation \cite{conf:cora}.

\begin{definition}\label{df:contzono}
(\textbf{Constrained Zonotope} \cite{scott2016constrained}) An $n$-dimensional constrained zonotope is defined by
\begin{equation}\label{eq:conszono}
 \mathcal{C} = \setdef[x\in\mathbb{R}^{n}]{x=c_{\mathcal{C}}+G_{\mathcal{C}} \beta, \ A_{\mathcal{C}} \beta=b_{\mathcal{C}}, \, \norm{\beta}_\infty\leq 1}, 
\end{equation}
where $c_{\mathcal{C}} \in \R^{n}$ is the center, $G_{\mathcal{C}}=\begin{bmatrix} g_{\mathcal{C}}^{(1)} \dots g_{\mathcal{C}}^{(\gamma_{\mathcal{C}})}\end{bmatrix}$ $\in$ $\R^{n \times n_g}$ is the generator matrix and $A_{\mathcal{C}} \in $ $\R^{n_c \times n_g}$ and $b_{\mathcal{C}} \in \R^{n_c}$ denote the constraints. In short, we use the shorthand notation $\mathcal{C}= \langle{c_{\mathcal{C}},G_{\mathcal{C}},A_{\mathcal{C}},b_{\mathcal{C}}}\rangle$ for a constrained zonotope.
\end{definition}

For an $n$-dimensional interval with lower and upper bounds $\underline{l} \in \R^n$ and $\overline{l} \in \R^n$, respectively, we use a notation to represent it as a zonotope $\mathcal{Z} = \langle{\underline{l},\overline{l}}\rangle \subset \R^n$, where the center is given by $\tfrac{1}{2}(\underline{l}+\overline{l})$ and the generator matrix is $\diag{\tfrac{1}{2}(\overline{l} - \underline{l})}$.

\subsection{System Dynamics and Safety Assumptions}
The robot is modeled as a discrete-time, nonlinear control system with an unknown model, where the state at time $k \in \mathbb{N}$ is given by ${x_k} \in \mathcal{X} \subset \mathbb{R}^n$. The state space $\mathcal{X}$ is assumed to be compact. At each time step $k$, the input ${u_k}$ is selected from a zonotope $\mathcal{U} \subset \mathbb{R}^m$ which represents the set of all possible actions. Process noise is denoted by ${w_k} \in \mathcal{W} \subset \mathbb{R}^n$. The system dynamics, represented by the black-box function $f: \mathcal{X} \times \mathcal{U} \times \mathcal{W} \to \mathcal{X}$, are described by:
\begin{align}\label{eq:sys}
 x_{k+1} &= f(x_k,u_k) + w_k.
\end{align}
We further assume that $f$ is twice differentiable and Lipschitz continuous, implying the existence of a \emph{Lipschitz constant} $L^\star$ such that, for all ${z}_1, {z}_2 \in \mathbb{R}^{n+m}$ with ${z}_j = \begin{bmatrix} x_j^T & u_j^T
\end{bmatrix}^T$, the following holds:
$
\norm{f({z}_1) - f({z}_2)} \leq L^\star \norm{{z}_1 - {z}_2}.
$
The initial state of the system, ${x}_{0}$, is drawn from a compact set $\mathcal{X}_{0} \subset \mathbb{R}^n$. 

To ensure safety guarantees, we also incorporate the concept of failsafe maneuvers from mobile robotics \cite{kousik2020bridging,magdici2016fail}.
The dynamics $f$ are invariant to positional translation, and the robot can brake to a complete stop within $n_{brake} \in \mathbb{N}$ time steps, remaining stationary indefinitely. Unsafe regions of the state space, referred to as obstacles, are denoted by $\mathcal{X}_{obs} \subset \mathcal{X}$. We assume obstacles are static but vary between episodes, as this work focuses on single-agent navigation rather than predicting the motion of other agents. Reachability-based frameworks for handling dynamic environments and other agents' motion \cite{leung2020infusing,vaskov2019not} can extend the applicability of this work. Finally, we assume the robot can instantaneously sense all obstacles $\mathcal{X}_{obs}$ and represent them as a union of constrained zonotopes. In cases with sensing limitations, a minimum detection distance can be computed to ensure safety, based on the robot's maximum speed and braking distance \cite{kousik2020bridging}.



\subsection{Safety via Reachable Set Computation}
We ensure safety by computing the forward reachable set of our robot for a given motion plan and then adjusting the plan to ensure that the forward reachable set does not intersect with any obstacles. We define the reachable set as follows:

\begin{definition}
The reachable set $\mathcal{R}_{k}$ at time step $k$, subject to a sequence of inputs ${u}_{j} \in \mathcal{U} \subseteq \mathbb{R}^m$, noise ${w}_{j} \in \mathcal{W}$ for all $j \in \{0, \dots, k-1\}$, and initial set $\mathcal{X}_{0} \subseteq \mathbb{R}^n$, is defined as:
\begin{align}\begin{split}\label{eq:reachable_set_k}
 \mathcal{R}_{k} = \big\{&{x_k} \in \mathbb{R}^n \, \big|\ 
 x_{j+1} = f(x_j, u_j) + w_j,\ {x}_{0} \in \mathcal{X}_{0},\\ 
 &{u}_{j} \in \mathcal{U}_{j},\ \text{and}\ w_{j} \in \mathcal{W},\ \forall\ j = 0,\dots,k-1\big\}.
\end{split}\end{align}\label{def:R}
\end{definition}

Note that we treat the dynamics $f$ as a black box, which may be nonlinear and challenging to model. However, we aim to overapproximate the reachable set $\mathcal{R}_k$.

\begin{figure}[t]
 \centering
 \includegraphics[width=0.9\columnwidth]{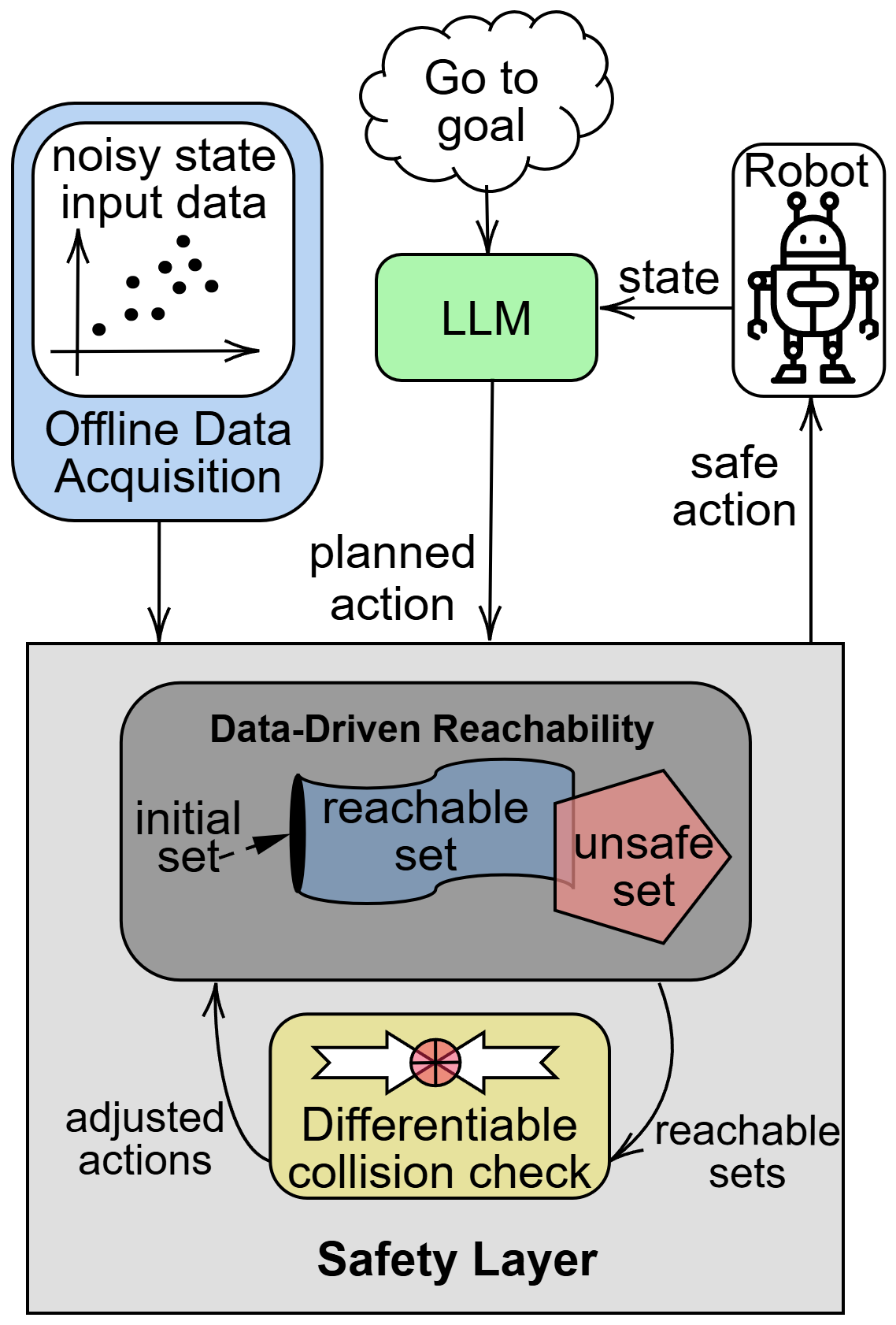}
 \caption{The proposed framework.}
 \label{fig:framework}
 \vspace{-6mm}
\end{figure}

\subsection{Problem Statement}

This research investigates the challenge of controlling a robotic system using text-based input commands. The system is modeled as a discrete-time, nonlinear control system with unknown dynamics, as detailed in~\eqref{eq:sys}. The primary goal is to achieve the safe operation of a black-box robot controlled by an LLM driven by textual inputs.

\begin{algorithm}[t]
\caption{Safe LLM-Based Robot Control.}
\label{alg:llmrobot}
\textbf{Input}: A textual prompt , maximum number of planning steps \( n_{\text{plan}} \), safe plan \( p_0 \), reaching distance radius \( r_{\text{goal}} \), initial robot position deviation \( \epsilon_{x_0} \), minimum distance for recognizing the obstacle \(d\)
\\\textbf{Output}: Safe deployment of a plan action
\begin{algorithmic}[1]
    \While{reaching distance \(\leq r_{\text{goal}}\)}
        \State \( \mathbf{p}_{k} = \) prompt the LLM for a plan
       
            \If{Distance to Obstacle  <= d}
            \State \( \hat{\mathcal{R}}_{k} = \langle\mathbf{x}_{k}, \epsilon_{x_0}\rangle \) \Comment{Initialize reachable set}
            \State \( (\hat{\mathcal{R}}_{j})_{j=k}^{k+n_{\text{plan}}} = \text{reach}(\hat{\mathcal{R}}_{k}, \mathbf{p}_{k}) \) \Comment{Use Alg.~\ref{alg:LipReachability}}
            \If{any \( \hat{\mathcal{R}}_{j} \cap \mathcal{X}_{\text{obs}} \neq \emptyset \)}
                \State \textbf{try:} \( \mathbf{p}_{k} = \text{adjust}(\mathbf{p}_{k}, \mathcal{X}_{\text{obs}}) \) \Comment{Use Alg.~\ref{alg:adjust}}
                \State \textbf{catch:} Use backup safe plan; \textbf{continue}
            \EndIf
        \EndIf

        \State \( \mathbf{u}_{k} = \) Get the first (safe) plan from \( \mathbf{p}_{k} \)
        \State Send the safe plan to the robot
    \EndWhile
\end{algorithmic}
\end{algorithm}

\section{Safe LLM-Controlled Robots} \label{sec:main}
Inspired by \cite{selim2022safe,amos2017optnet,chung2021constrained}, we propose providing safety guarantees for LLM-controlled robots through reachability analysis. As illustrated in Figure~\ref{fig:framework}, the proposed approach leverages offline-collected data (blue, top left) to enhance robot trajectory planning. The process follows a data-driven receding-horizon approach, forming an anti-clockwise loop per planning cycle: starting with the LLM, passing through the safety layer, and reaching the robot. Initially, the LLM takes a text prompt along with the robot’s current state (white, upper right) to generate a planned action. Then, the safety layer (grey, bottom) refines this action by ensuring safety through data-driven reachability analysis and a differentiable collision-checking method that evaluates the robot’s reachable sets. If a collision-free action cannot be determined, a failsafe maneuver is executed. Finally, the adjusted safe action is sent to the robot for execution. 

Algorithm~\ref{alg:llmrobot} ensures safe LLM-driven robot control by integrating reachability analysis. It starts by initializing key parameters—a prompt, a safe initial plan, the maximum planning horizon, goal distance, and initial state deviation. In each iteration, the LLM generates a plan. Then, the reachable sets for future steps based on the plan are computed. If any reachable set intersects with an obstacle, the algorithm attempts to adjust the plan; if adjustment fails, a failsafe maneuver is executed. The first safe action from the verified plan is then applied, and this loop continues until the robot safely reaches its goal. The first component of our proposed approach is the data-driven reachability analysis, which is explained next.

\subsection{Data-Driven Reachability Analysis}\label{subsec:data_driven_reachability_analysis}

Data-driven reachability analysis computes a reachable set that encompasses all possible robot locations derived from past trajectories, eliminating reliance on potentially inaccurate models that could compromise intended safety objectives. The intersection between the reachable set, computed for a plan ${p}_k = ({u}_j)_{j=k}^{n_{plan}}$, and the obstacles is analyzed to assess compliance with the intended safety requirements. The reachability analysis uses Algorithm~\ref{alg:LipReachability}, adapted from \cite{alanwar2021data}. This algorithm overapproximates the reachable set, as defined in \eqref{eq:reachable_set_k}, by calculating a zonotope ${\hat{\mathcal{R}}}_j \supseteq \mathcal{R}_j$ for each time step in the current plan.

\begin{algorithm}[t]
\caption{Data-Driven Reachability Analysis~\cite{alanwar2021data}.}\label{alg:LipReachability}
\textbf{Input}: initial reachable set \( \hat{\mathcal{R}}_{0} \), inputs \( ({u}_{j}^{*})_{j=k}^{k+n_{\text{plan}}} \), state/input data \( ({X}_{-}, {X}_{+}, {U}_{-}) \), noise zonotope \( \mathcal{W}{=}\langle{c}_{{w}}, {G}_{{w}}\rangle \), Lipschitz constant \( L^{*} \), covering radius \( \delta \), small number $\epsilon$
\\\textbf{Output}: overapproximated reachable sets $(\hat{\mathcal{R}}_{j})_{j=k}^{k+n_{\text{plan}}}$
\begin{algorithmic}[1]
\State \( Z_{\epsilon} = \langle{0}, \operatorname{diag}(({L}^{*})_{1}(\delta)_{1}/2, \dots, ({L}^{*})_{n}(\delta)_{n}/2)\rangle \) \label{ln:alglipZeps}
\For{\( j = k : (k + n_{\text{plan}}) \)}
    \State \( {M}_{j} = ({X}_{+} - {c}_{{w}}) \begin{bmatrix} {1} \\ {X}_{-} - 1 \otimes {x}_{j}^{*} \\ {U}_{-} - 1 \otimes {u}_{j}^{*} \end{bmatrix} \)\label{ln:alglipMtilde}
    \State \( \underline{ {I}} = \min_{j} \left( ({X}_{+})_{:,j} - {M}_{j} \begin{bmatrix} 1 \\ ({X}_{-})_{:,j} - {x}_{j}^{*} \\ ({U}_{-})_{:,j} - {u}_{j}^{*} \end{bmatrix} \right) \)
    \State \( \bar{{I}} = \max_{j} \left( ({X}_{+})_{:,j} - {M}_{j} \begin{bmatrix} 1 \\ ({X}_{-})_{:,j} - {x}_{j}^{*} \\ ({U}_{-})_{:,j} - {u}_{j}^{*} \end{bmatrix} \right) \)
    \State \( \mathcal{Z}_{L} = \langle\underline{ {I}}, \bar{{I}}\rangle - \mathcal{W} \) 
    \State \( \mathcal{U}_{j} = \langle{u}_{j}^{*}, \epsilon\rangle \)
    \State \( \hat{\mathcal{R}}_{j+1} = {M}_{j} (1 \times (\hat{\mathcal{R}}_{j} - {x}_{j}^{*}) \times (\mathcal{U}_{j} - {u}_{j}^{*})) + \mathcal{W} + \mathcal{Z_{L}} + \mathcal{Z_{\epsilon}} \)
\EndFor

\Return \( (\hat{\mathcal{R}}_{j})_{j=k}^{k+n_{\text{plan}}} \) \Comment{overapproximates~\eqref{eq:reachable_set_k}}
\end{algorithmic}
\end{algorithm}

Our approach leverages noisy trajectory data collected offline from the black-box system model, while online data is reserved exclusively for training the policy and environment model. We utilize ${q}$ input-state trajectories. For efficient matrix operations in Algorithm~\ref{alg:LipReachability}, we structure the data into the following matrices, which are written for a single trajectory with length $T$ to ease the notations. 
\begin{subequations}\label{eq:data}
\begin{align}
 {X}_- &= \left[{{x}}_0, \dots, {{x}}_{T-1} \right], \\
 {X}_+ &= \left[{{x}}_1, \dots, {{x}}_{T} \right], \\
 {U}_- &= \left[{u}_0, \dots, {u}_{T-1} \right].
\end{align}
\end{subequations}

We estimate the Lipschitz constant of the dynamics from the dataset $({X}_-,{X}_+,{U}_-)$ following the approach in \cite[Remark 1]{alanwar2021data}. Additionally, we define a data covering radius $\delta$ such that, for any point ${z}_1 \in \mathcal{X} \times \mathcal{U}$, there exists a point ${z}_2 \in \mathcal{X} \times \mathcal{U}$ with $\norm{{z}_1 - {z}_2}_2 \leq \delta$. We assume sufficient offline data is available {a priori} to upper-bound $L^\star$ and lower-bound $\delta$, with these bounds holding consistently across offline collection and online execution, consistent with prior work \cite{koller2018learning, alanwar2021data}. To mitigate overconservatism in the reachable set, we compute distinct $(L^\star)_i$ and $(\delta)_i$ for each dimension, enabling a refined Lipschitz zonotope $\mathcal{Z}_\epsilon$ (see Line~\ref{ln:alglipZeps} of Algorithm~\ref{alg:LipReachability}).

\subsection{Adjusting Unsafe Actions}

After the LLM generates a plan ${p}_k$, the safety layer refines it by verifying the intersection of the plan’s reachable sets with unsafe regions. This adjustment process depends only on the unsafe sets surrounding the robot. If all actions in the plan are deemed safe, it is executed in the environment; otherwise, we seek a safe alternative. Rather than relying on inefficient random sampling in expansive action spaces, we employ gradient descent to modify the plan, ensuring reachable sets avoid collisions and incorporate a failsafe maneuver.

Unsafe actions are adjusted via Algorithm~\ref{alg:adjust}. If the algorithm fails to converge within one-time step, it halts, and the robot reverts to the prior safe plan. The process iterates over each action in ${p}$, performing these steps: compute the reachable set for all subsequent steps using Algorithm~\ref{alg:LipReachability}, check for collisions, and, if detected, calculate the collision-check gradient and apply projected gradient descent. A safe plan is returned upon convergence; otherwise, it is flagged as unsafe. The final plan must embed a failsafe maneuver.

Collision checking between reachable and unsafe sets, both modeled as constrained zonotopes, proceeds as follows. For two zonotopes $\mathcal{Z}_1 = \langle{{c}_1, {G}_1, {A}_1, {b}_1}\rangle$ and $\mathcal{Z}_2 = \langle{{c}_2, {G}_2, {A}_2, {b}_2}\rangle$, their intersection is $\mathcal{Z}_\cap = \mathcal{Z}_1 \cap \mathcal{Z}_2 = \langle{{c}_{\cap}, {G}_{\cap}, {A}_{\cap}, {b}_{\cap}}\rangle$ \cite{scott2016constrained}, defined by:
\begin{align}
\mathcal{Z}_\cap = \Bigg\langle {c}_1, [{G}_1, {0}], \begin{bmatrix}
 {A}_1 & {0} \\
 {0} & {A}_2 \\
 {G}_1 & -{G}_2
 \end{bmatrix}, \begin{bmatrix}
 {b}_1 \\ {b}_2 \\ {c}_2 - {c}_1
 \end{bmatrix} \Bigg\rangle.
\end{align}
We determine if $\mathcal{Z}_1 \cap \mathcal{Z}_2$ is empty by solving the linear program \cite{scott2016constrained}:
\begin{align}\label{prog:collision_check}
 v\opt = \min_{{z}, v} \left\{ v \mid {A}_{\cap} {z} = {b}_{\cap}, \ |{z}| \leq v \right\},
\end{align}
where $|{z}|$ is elementwise; $\mathcal{Z}_\cap$ is nonempty if and only if $v\opt \leq 1$.

\begin{algorithm}[t]
\caption{Adjusting Unsafe Actions~\cite{selim2022safe}.}
\label{alg:adjust}

\textbf{Input}: plan \( p_k = (u_j)_{j=k}^{k+n_{\text{plan}}} \), obstacles \( \mathcal{X}_{\text{obs}} \), initial reachable set \( \hat{\mathcal{R}}_k \), step size \( \gamma \), time steps required to stop \( n_{\text{break}} \)
\\\textbf{Output}: safe plan
\begin{algorithmic}[1]
\State \( p_{\text{safe}} = p_k \) \textbf{initialize with given plan}
\For{\( j = k : (k + n_{\text{break}}) \)}
        \State \( {(\hat{\mathcal{R}}_i)}_{i=k}^{k+n_{\text{plan}}} = \text{reach} \left(\hat{\mathcal{R}}_j, p_{\text{safe}} \right) \) \Comment{use Alg.~\ref{alg:LipReachability}}
        \If{ \( {(\hat{\mathcal{R}}_i)}_{i=k}^{k+n_{\text{plan}}} \cap \mathcal{X}_{\text{obs}} \neq \emptyset \) } \Comment{using~\eqref{prog:collision_check} }

            \State \( u_j = \text{proj}_{I_j} \left( u_j - \gamma \nabla_{u_j} v^* \right) \) \Comment{using~\eqref{eq:collision_chain_rule}}
        \EndIf
\EndFor
\If{ \( {(\hat{\mathcal{R}}_i)}_{i=k}^{k+n_{\text{plan}}} \cap \mathcal{X}_{\text{obs}} = \emptyset \) }
    \State  \textbf{try:} \( p_{\text{safe}} = (u_j)_{j=k}^{k+n_{\text{plan}}} \) \Comment{apply safe plan}
    \State\textbf{catch:}  Override with the backup safe plan \Comment{failed to project the plan}
\EndIf
\end{algorithmic}
\end{algorithm}


To avert collisions, we use gradient descent to adjust the reachable sets ${\hat{R}}_k$. Let $\hat{{c}}_k$ be the center of ${\hat{R}}_k$, which, per Algorithm~\ref{alg:LipReachability}, depends on ${u}_0, \dots, {u}_{k-1}$. Given an optimal solution $({z}\opt, v\opt)$ to \eqref{prog:collision_check} for ${\hat{R}}_k$ and an unsafe set, collision avoidance requires $v\opt > 1$ \cite{scott2016constrained}. We compute the gradient $\nabla_{\RLaction_k} v\opt$ with respect to the input action via a chain rule recursion:
\begin{align}\label{eq:collision_chain_rule}
 \nabla_{\RLaction_{h}} v\opt = \nabla_{\hat{{c}}_k} v\opt
 \nabla_{\hat{{c}}_{k - 1}} \hat{{c}}_k
 \left( \prod_{j=h + 2}^{k - 1} \nabla_{\hat{{c}}_{j - 1}} \hat{{c}}_j \right)
 \nabla_{\RLaction_{h}} \hat{{c}}_{h + 1},
\end{align}
where $h = k - i$. The gradients of $\hat{{c}}_k$ are:
\begin{subequations}
\begin{align}
 \nabla_{\hat{{c}}_{k - 1}} \hat{{c}}_k &= ({M}_{k - 1})_{(1:1+n),(1:1+n)}, \label{eq:collision_chain_rule_for_input_zono1} \\
 \nabla_{u_{k - 1}} \hat{{c}}_k &= ({M}_{k-1})_{:,(n+1:n+1+m)}, \label{eq:collision_chain_rule_for_input_zono2}
\end{align}
\end{subequations}
with ${M}_{k-1}$ from Algorithm~\ref{alg:LipReachability}, Line~\ref{ln:alglipMtilde}, and $n$ and $m$ as state and action dimensions. After gradient descent using $\nabla_{{u}_k} v\opt$, we project ${u}_k$ onto the feasible control set: $\proj_{{U}_k}({u}_k) = \argmin_{{v} \in {U}_k} \norm{{u}_k - {v}}_2^2$. Since the resulting controls may remain unsafe, we recheck the final reachable sets in Algorithm~\ref{alg:adjust}.

\subsection{Safety Guarantees}

We conclude by formalizing the safety guarantees for LLM-controlled robots.

\begin{theorem}\label{thm:theorem}
Assume the robot and the environment satisfy the conditions in Section~\ref{sec:prelim}, and the robot starts in a safe state at $k = 0$. Given an input text command to the LLM to let the robot go to a target, then Algorithm~\ref{alg:llmrobot} guarantees that the robot remains safe; if at each $k > 0$, the LLM generates a plan ${p}_k$ and it is adjusted using Algorithm~\ref{alg:adjust}.
\end{theorem}
\begin{proof}
We prove this by induction. At $k = 0$, the robot can apply $u_{brake}$ to stay safe indefinitely. Assume a safe plan exists at time $k \in \N$. If Algorithm~\ref{alg:adjust} outputs an unsafe plan, the robot defaults to the prior safe plan; otherwise, the new plan is safe due to three properties: (1) Algorithm~\ref{alg:LipReachability} ensures the reachable set overapproximates the true set, as process noise is bounded by a zonotope \cite[Theorem 2]{alanwar2021data}; (2) Algorithm~\ref{alg:adjust}’s collision check reliably detects overlaps \cite{scott2016constrained}; and (3) the adjusted plan enforces a stop after $n_{plan}$ steps, embedding a failsafe maneuver.
\end{proof}

\section{Case Studies}\label{sec:eval}
\begin{figure*}[h]
 \centering
 \begin{subfigure}[cb]{0.32\textwidth}
 \centering
 \includegraphics[height=0.23\textheight]{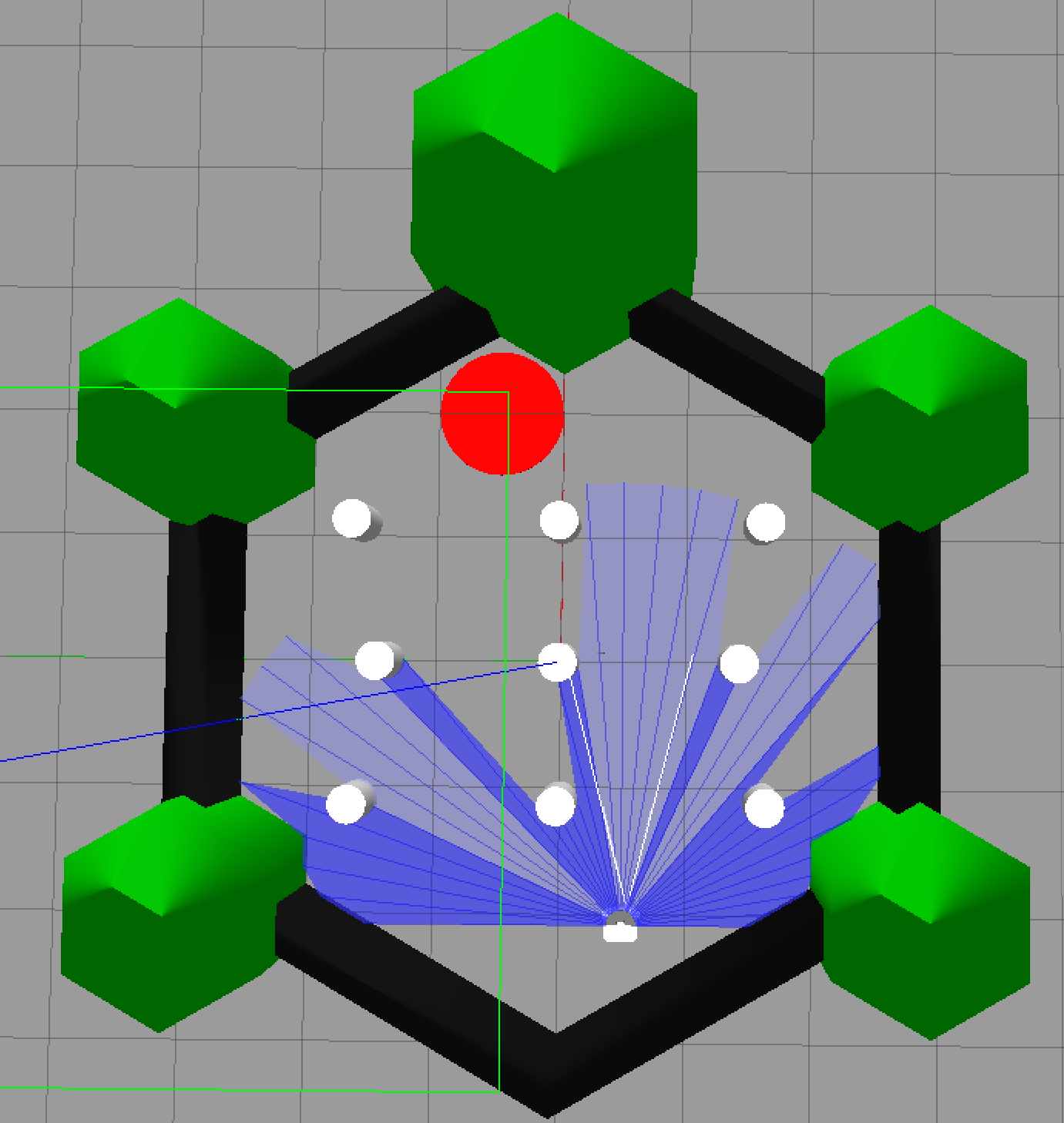}
 \caption{TurtleBot3 world.}
 \label{fig:world}
 \end{subfigure}
 \hfill
 \begin{subfigure}[bc]{0.33\textwidth}
 \centering
 \includegraphics[height=0.23\textheight]{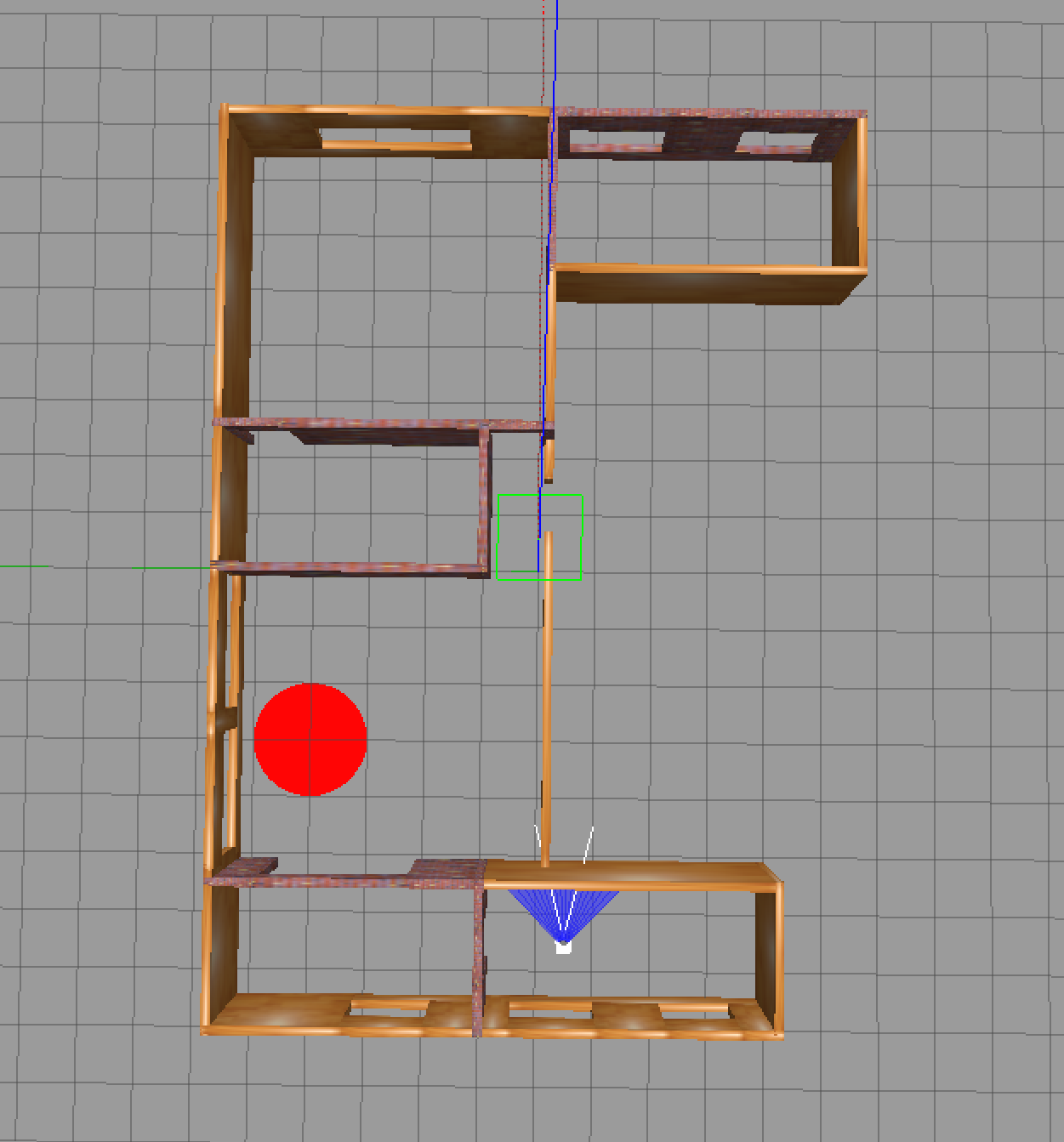}
 \caption{TurtleBot3 house.}
 \label{fig:house}
 \end{subfigure}
\hfill
 \begin{subfigure}[bc]{0.33\textwidth}
 \centering
 \includegraphics[height=0.23\textheight]{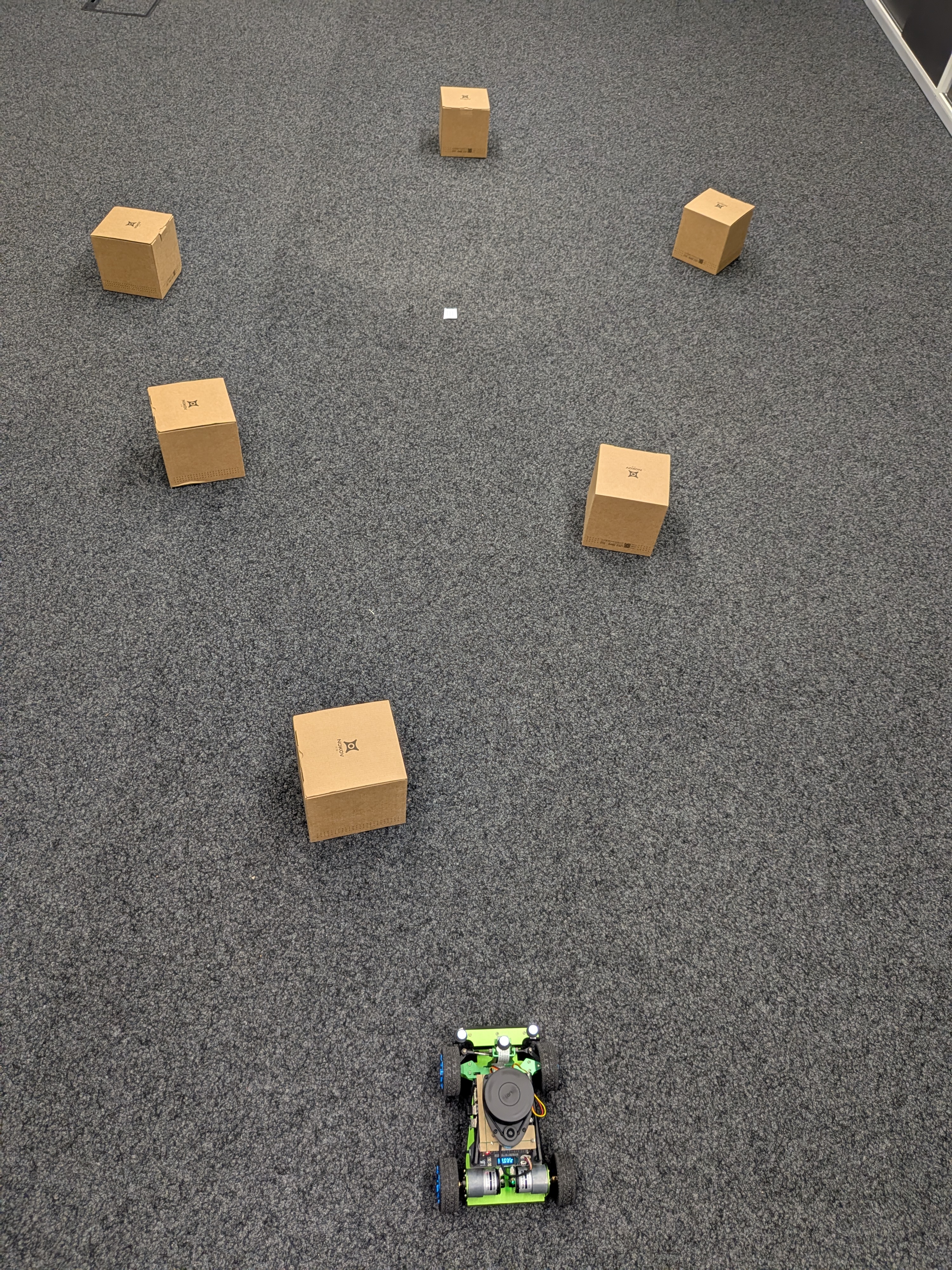}
 \caption{JetRacer in the CPS lab.}
 \label{fig:JRP}
 \end{subfigure}
 \caption{Evaluation environments where in a and b the TurtleBot3 robot is the white rectangle and the red circle is the target, for c the target is the white spot.}
 \label{fig:turtlebot3_environments}
 \vspace{-6mm}
\end{figure*}


We evaluate the proposed method for LLM-driven robotic control using two case studies: a 2D differential robot (TurtleBot3) and a small autonomous vehicle (JetRacer). To conduct the data-driven reachability analysis, we first collect 600-step noisy input/output data offline in an empty environment.

The TurtleBot3 is simulated in two distinct environments using Gazebo and ROS, as illustrated in Figure~\ref{fig:turtlebot3_environments}: In the world environment (Figure~\ref{fig:world}), the robot navigates through cubic obstacles, testing its ability to perform obstacle avoidance. In the house environment (Figure~\ref{fig:house}), the system evaluates indoor navigation and object detection within a furnished space. A snippet of the prompt used for LLM control is shown in Figure~\ref{fig:prompt}. 

The JetRacer operates in the Cyber-Physical Systems (CPS) lab, where it is commanded via LLM to reach a designated target position, as depicted in Figure~\ref{fig:JRP}.

\begin{figure*}[t]
 \centering
 \includegraphics[width=0.55\textwidth]{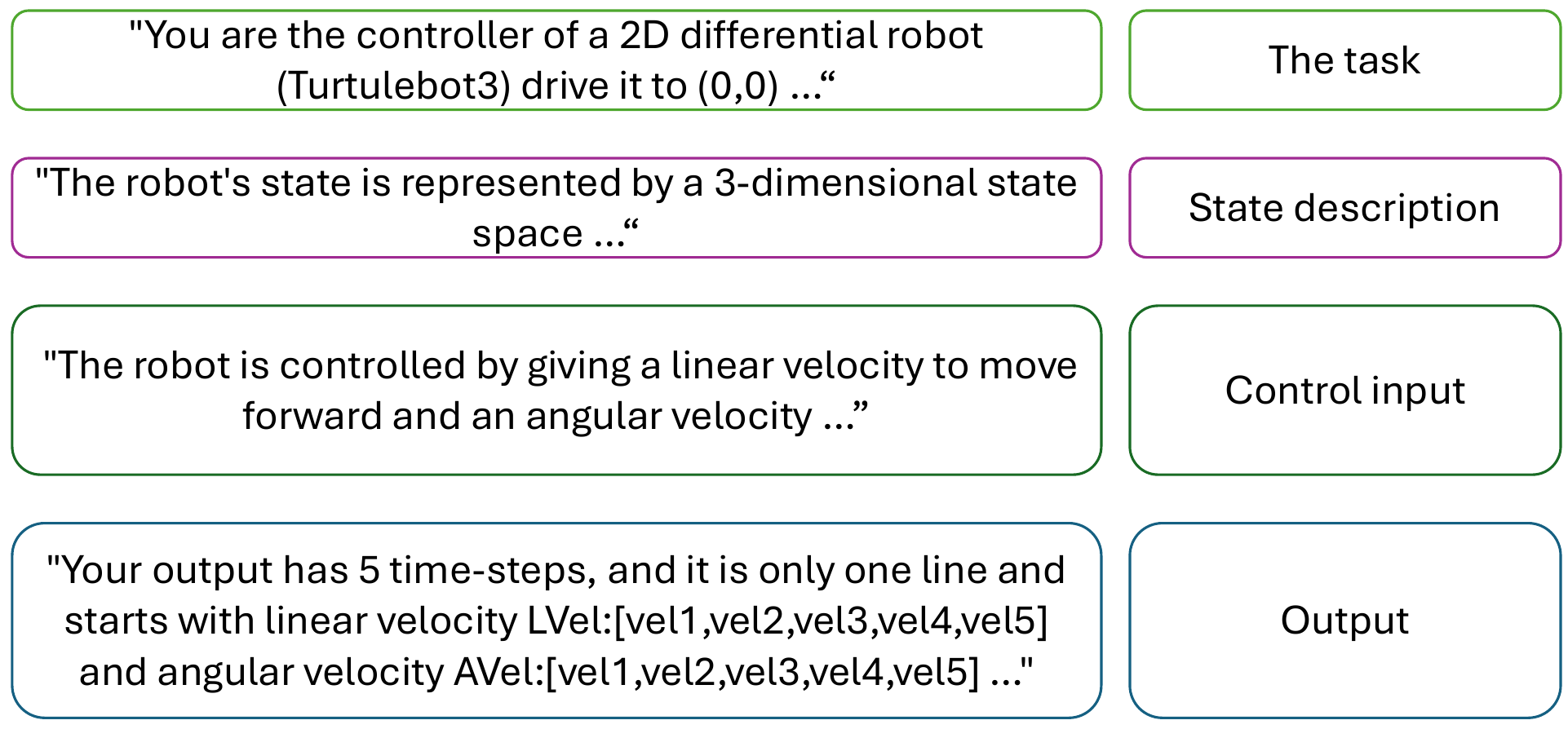}
 \caption{Snippet of the prompt.}
 \label{fig:prompt}
 \vspace{-6mm}
 \end{figure*}

\subsection{Black-Box Dynamical Model}
We use \( {x} = [p_x, p_y, \cos(\psi), \sin(\psi)] \) and \( u = [v, \omega] \) as state and input vectors respectively. \( v \) is the linear velocity and \( \omega \) is the angular velocity. \(p_x,p_y\) and \(\psi\) are the robot position and orientation, respectively.

\subsection{Implementation Details and Results}
We have used the OpenAI "GPT-4o" model with a temperature parameter set to ($0.1$). Increasing this parameter to near $1$ allows the LLM model to explore more possibilities. We have considered 3 steps planning horizon (5 steps for JetRacer). The plan is generated by ChatGPT. This model is limited to 500 Requests Per Minute (RPM), which is around 8.3 Hz, and 30,000 Tokens Per Minute (TPM), which is equal to 500 tokens per second. Assuming four english character as a token and each number as around 3.5 tokens, our prompt is about 230 tokens. Given the information about the OpenAI model used in our simulation, we could reach a maximum of 1.5 Hz update rate.

In Table~\ref{tab:reachability_analysis}, the performance of the data-driven reachability filter is evaluated using the TurtleBot3. The execution time and frequency of the reachability analysis were computed for a single time step, considering different obstacle configurations and planning horizons. As the number of obstacles and the planning horizon increase, the computational cost and execution time also increase.

\begin{table}[t]
\centering
\caption{Performance of data-driven reachability analysis across varying conditions. There is no collision in all cases.}
\label{tab:reachability_analysis}
\begin{tabular}{@{}cccc@{}}
\toprule
\begin{tabular}[c]{@{}c@{}}Number of\\ obstacles\end{tabular} & \begin{tabular}[c]{@{}c@{}}Number of\\ plans\end{tabular} & \begin{tabular}[c]{@{}c@{}}Execution\\ time (s)\end{tabular} & \begin{tabular}[c]{@{}c@{}}Frequency\\ (Hz)\end{tabular} \\ \midrule
 & 3 & 0.04 & 25 \\
3 & 5 & 0.1 & 10 \\
 & 10 & 0.16 & 6 \\ \cmidrule(r){1-1}
 & 3 & 0.08 & 12.5 \\
5 & 5 & 0.16 & 6 \\
 & 10 & 0.22 & 4.5 \\ \bottomrule
\end{tabular}
\vspace{-0.3cm}
\end{table}

The robot's input linear velocity is limited to \( [0.0, 0.1] \, \text{m/s} \) in the first step and \( [0.0, 0.5] \, \text{m/s} \) for the rest, while the angular velocity is limited to \( [-0.5, 0.5] \, \text{rad/s} \) in the reachability analysis. The TurtleBot3 is equipped with wheel encoders and a planar LiDAR that generates 19 range measurements between \( -45^\circ \) and \( 45^\circ \) for house environments and 37 range measurements between \( -90^\circ \) and \( 90^\circ \) for world environments.

In the prompt, the current location and orientation of the robot, along with LiDAR information and the distance to the goal location, are provided. However, we set the linear velocity bound to \( [0.1, 0.5] \, \text{m/s} \) in the prompt to avoid zero output by ChatGPT, based on our experiments. All TurtleBot3's codes run on a computer with an Intel i5 1235U CPU and 8 GB memory (RAM). 
\begin{table}[t]
\centering
\caption{Comparison of using safety filters and planners.}
\label{tab:comparison}
\begin{tabular}{@{}cccc@{}}
\toprule & Our approach & LLM-AISF & RL-SAILR~\cite{wagener2021safe} \\ 
\midrule
\begin{tabular}[c]{@{}c@{}}Formal safety\\ guarantee\end{tabular} & Yes & No & No \\
\begin{tabular}[c]{@{}c@{}}Minimum distance\\ to obstacle (m)\end{tabular} & 0.1 & 0.5 & - \\
\begin{tabular}[c]{@{}c@{}}Plan horizon\\ (steps)\end{tabular} & 3 & 1 & 1 \\
Pre-training & No & No & Yes \\
collision & No & Yes & Yes \\ \bottomrule
\end{tabular}
\vspace{-0.5cm}
\end{table}


In Table~\ref{tab:comparison}, we compare the data-driven reachability-based safety filter used in this work with the Advantage-based Intervention Safety Filter (AISF), where an LLM serves as the controller within the Safe Advantage-based Intervention for Learning Policies with Reinforcement (SAILR) framework~\cite{wagener2021safe}.

When comparing the safety filters alone, our approach provides a formal safety guarantee. AISF, on the other hand, checks safety only one step ahead, requiring the agent to maintain a sufficient distance from obstacles to prevent collisions. In contrast, our safety filter evaluates safety over a horizon of multiple time steps, allowing the agent to operate closer to obstacles while ensuring safety. Additionally, we employed the current LLM model without any pre-training, whereas reinforcement learning methods require an unsupervised learning process to train the agent. Our approach also guarantees collision avoidance, which LLM-AISF fails to do in some cases. However, our approach comes at a higher computational cost than LLM-AISF, presenting a trade-off between safety and efficiency.

Based on our experiments, ChatGPT is sensitive to prompt wording. Different prompts might lead to different performances and even failure to plan successfully.
Since utilizing a stateless API and zero-shot learning allows each request to be fully defined within its prompt, we can efficiently process reachability analysis results independently, enabling consistent and high-quality performance without the need for prior feedback.

\subsection{Application to a Small Vehicle}
Our framework is also applied to a JetRacer, as shown in Figure~\ref{fig:JRP} and Figure~\ref{fig:JRW}, to further validate its effectiveness. This experiment took place within a NOKOV motion capture-enabled environment consisting of 10 9-megapixel NOKOV cameras, providing precise real-time tracking of the JetRacer's position and orientation. Using LLM-generated text prompt, the JetRacer successfully demonstrated obstacle avoidance while reaching its designated goal. The experiment involved multiple obstacles, where the LLM provided a 5-step planning horizon to navigate the JetRacer safely. Leveraging the same data-driven reachability analysis, the system ensured that all reachable sets remained collision-free, achieving robust safety guarantees while maintaining adaptability to the JetRacer's dynamics. This application underscores the versatility of our approach across different robotic platforms and operational conditions.

\begin{figure}[t]
 \centering
 \includegraphics[width=0.6\columnwidth]{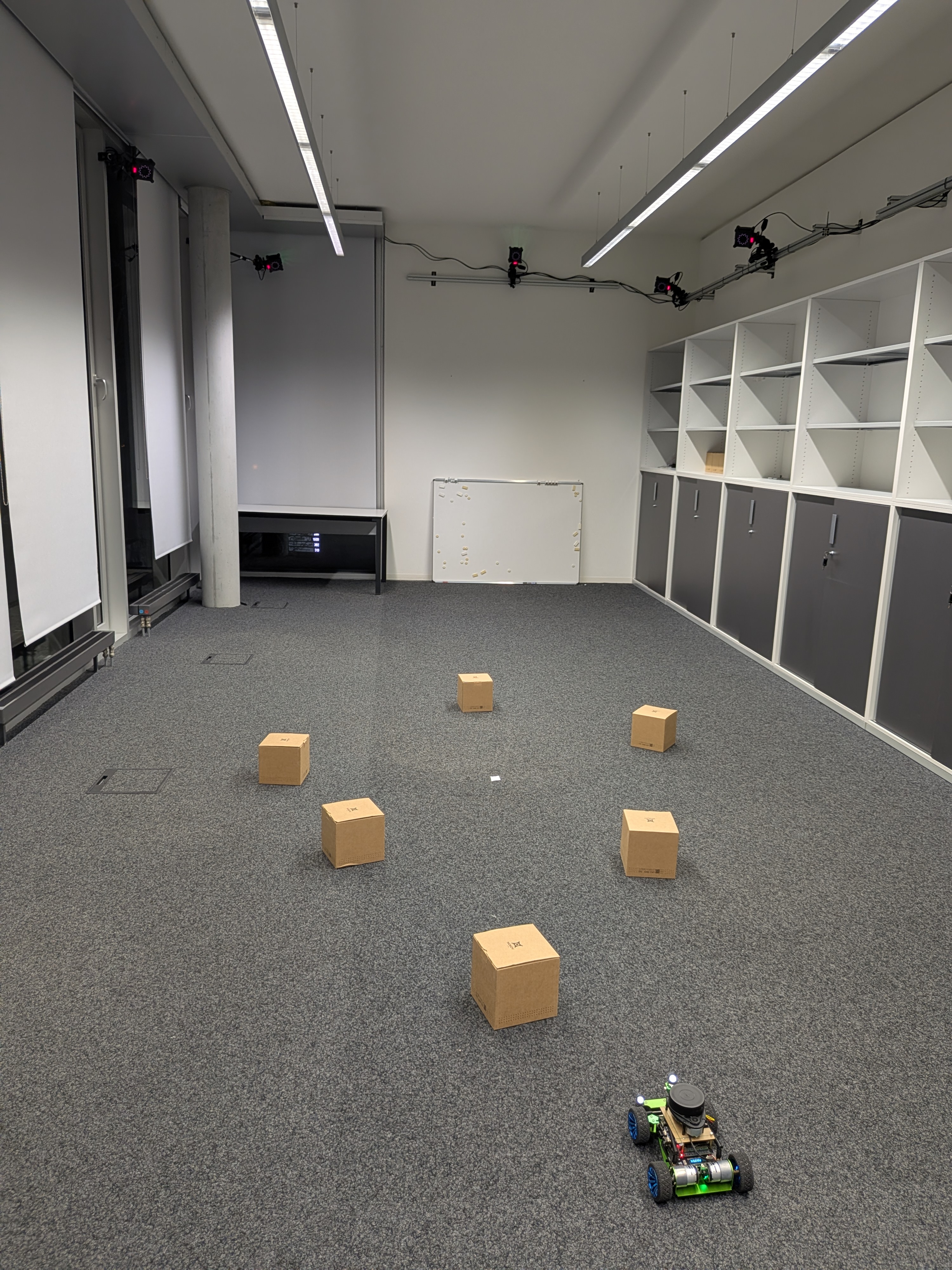}
 \caption{JetRacer with motion capture system, the white point is the goal, and the boxes are the obstacles.}
 \label{fig:JRW}
 \vspace{-6mm}
\end{figure}
\section{Conclusion}\label{sec:con}

This paper presents a novel safety assurance framework for LLM-controlled robotic systems, addressing the critical challenge of ensuring safe operation in dynamic and unpredictable environments. By integrating LLMs with zero-shot learning capabilities and data-driven reachability analysis, we provide a principled approach to verifying and adjusting LLM-generated plans without relying on potentially inaccurate analytical models. Our framework leverages offline trajectory data to compute overapproximated reachable sets, ensuring that all possible system trajectories remain within safe operational limits. The provided case studies highlight its ability to mitigate risks associated with the probabilistic nature of LLMs, achieving formal safety guarantees while maintaining adaptability to unseen tasks. 

Future work includes integrating feedback mechanisms into the LLM control loop, offering a promising avenue for enhancing adaptability. While stateless API communication enables zero-shot learning, it also limits the ability to incorporate direct feedback from reachability analysis results. To address this, developing a stateful interaction model or fine-tuning LLMs with safety-aware training data could significantly improve plan generation quality, reducing the need for extensive manual adjustments.

\bibliographystyle{IEEEtran}
{\small
\bibliography{IEEEabrv,ref}
}

\end{document}